\documentclass[11pt]{article}

\usepackage[utf8]{inputenc}

\usepackage{graphicx}
\usepackage{subfigure}
\usepackage{pgfplots}
\usetikzlibrary{shapes,arrows}

\usepackage{amsthm}

\usepackage{amsmath}
\usepackage{amsfonts}
\usepackage{amssymb}
\usepackage{array}
\usepackage{mathabx} 
\usepackage{algorithm}
\usepackage{algorithmic}
\usepackage{mathrsfs}

\usepackage{cite}
\usepackage{enumerate}
\usepackage{url}
\usepackage{lipsum}
\usepackage[author={Jeremy}]{pdfcomment}

\usepackage{dsfont}
\usepackage{latexsym}
\usepackage{hhline}
\usepackage{color}
\usepackage{textcomp}

\newtheorem{lemma}{Lemma}
\newtheorem{theorem}{Theorem}
\newtheorem{corollary}{Corollary}
\newtheorem{definition}{Definition}
\newtheorem{remark}{Remark}

\newtheorem{example}{Example}
\DeclareMathOperator{\rank}{rank} 
  
\DeclareMathOperator{\spann}{span} 
\DeclareMathOperator{\dimn}{dim}  
\DeclareMathOperator{\sparkn}{spark}  

\newcommand{\jc}{\color{black}}
\newcommand{\fin}{\color{black}}

\newcommand{\vb}{\mathbf{b}}

\newcommand{\matr}[1]{\mathbf{#1}}

\makeatletter
\newcommand{\argmin}{\mathop{\operator@font argmin}}

\definecolor{gris}{gray}{0.90}
\definecolor{gris25}{gray}{0.90}
\definecolor{americanrose}{rgb}{1.0, 0.01, 0.24}
\definecolor{bostonuniversityred}{rgb}{0.8, 0.0, 0.0}
\definecolor{shamrockgreen}{rgb}{0.0, 0.62, 0.38}
\definecolor{selectiveyellow}{rgb}{1.0, 0.73, 0.0}
\definecolor{royalblue}{rgb}{0.25, 0.41, 0.88}
\definecolor{ashgrey}{rgb}{0.7, 0.75, 0.71}
\definecolor{burgundy}{RGB}{159,29,53}
\definecolor{darkgreen}{RGB}{18,53,26}
\definecolor{lightblue}{RGB}{102,217,255}
\definecolor{fakeorange}{RGB}{255,140,102}
\definecolor{arylideyellow}{rgb}{0.91, 0.84, 0.42}
\definecolor{bananayellow}{rgb}{1.0, 0.88, 0.21}
\definecolor{gris_f}{gray}{0.35}
\definecolor{bordure}{rgb}{0.09,0.17,0.68}
\definecolor{aquamarine}{rgb}{0.5, 1.0, 0.83}
\definecolor{apricot}{rgb}{0.98, 0.81, 0.69}
\definecolor{babyblue}{rgb}{0.54, 0.81, 0.94}
\definecolor{uipoppy}{RGB}{225, 64, 5}
\definecolor{uipaleblue}{RGB}{96,123,139}
\definecolor{uiblack}{RGB}{0, 0, 0}
\definecolor{decoda}{RGB}{0,153, 0}
\definecolor{lightgreen}{rgb}{0.56, 0.93, 0.56}
\definecolor{blue_f}{rgb}{0.2, 0.2, 0.6}
\definecolor{cinnamon}{rgb}{0.82, 0.41, 0.12}
\definecolor{darkpastelgreen}{rgb}{0.2, 0.75, 0.24}
\definecolor{drab}{rgb}{0.59, 0.44, 0.09}

\ifpdf
\hypersetup{
  pdftitle={Identifiability of Complete Dictionary Learning},
  pdfauthor={J.E. Cohen and N. Gillis}
}
\fi

\title{
  Identifiability of Complete Dictionary Learning\thanks{The authors acknowledge the support by the Fonds de la recherche scientifique-FNRS (incentive grant for scientific research no F.4501.16). NG also acknowledges the support by the European Research Council (ERC starting grant no 679515), and by the Fonds de la Recherche Scientifique - FNRS and the Fonds Wetenschappelijk Onderzoek - Vlanderen (FWO) under EOS Project no O005318F-RG47.}}
  
  \date{}

\author{
Jeremy E. Cohen\thanks{CNRS, Universit\'e de Rennes, Inria, IRISA Campus de Beaulieu, 35042 Rennes, France.} 
\and 
Nicolas Gillis\thanks{Corresponding author. Email: nicolas.gillis@umons.ac.be. 
Department of Mathematics and Operational Research, Faculté polytechnique, Université de Mons, Rue de Houdain 9, 7000 Mons, Belgium.}
}

\begin{document}
\maketitle

\begin{abstract}
  Sparse component analysis (SCA), \jc also known as \fin complete dictionary learning, is the following problem: Given
        an input matrix $M$ and an integer $r$, find a dictionary $D$ with $r$
        columns and a matrix $B$ with $k$-sparse columns (that is, each column
        of $B$ has at most $k$ non-zero entries) such that $M \approx DB$. A
        key issue in SCA is identifiability, that is, characterizing the
        conditions under which $D$ and $B$ are essentially unique (that is,
        they are unique up to permutation and scaling of the columns of $D$ and
        rows of $B$).  Although SCA has been vastly investigated in the last
        two decades, only a few works have tackled this issue in the
        deterministic scenario, and no work provides reasonable bounds in the
        minimum number of  samples (that is, columns of $M$)  that leads to
        identifiability. In this work, we provide new results in the
        deterministic scenario when the data has a low-rank structure, that is,
        when $D$ \jc is (under)complete. \fin While previous bounds feature a combinatorial
        term $r \choose k$, we \jc exhibit a sufficient condition
        \fin involving $\mathcal{O}(r^3/(r-k)^2)$ samples \jc that yields
        \fin an essentially unique decomposition, as long as
        these data points are well spread among the subspaces spanned by $r-1$
        columns of $D$. \jc We also exhibit a necessary lower bound on the number of
        samples that contradicts previous results in the literature when $k$
        equals $r-1$. \fin Our bounds provide a drastic improvement compared to the state of the art, and imply for example
        that for a fixed proportion of zeros (constant and independent of $r$,
        e.g., 10\% of zero entries in $B$), one only requires $\mathcal{O}(r)$
        data points to guarantee identifiability. 
\end{abstract}


\section{Introduction}


In the last two decades, 
dictionary learning has had tremendous success in various fields such
as image processing, neuroimaging and remote sensing; 
see, e.g.,~\cite{elad2006image, mairal2009online, mairal2009supervised, rubinstein2010dictionaries, tosic2011dictionary, mairal2012task} 
and the references therein. 
In fact, many applications in source separation involve data expressed as a
combination of a few atoms from an appropriate but unknown basis. 
After the pioneer work of Olshausen and Field~\cite{Olshausen1997}, efforts
were put to derive conditions under which a `true' underlying dictionary and sparse
coefficients could be recovered with
certainty~\cite{Georgiev2005,aharon2006uniqueness} or
almost-surely~\cite{spielman2012exact, Gribonval2015}. Algorithmic
aspects of dictionary learning have also been extensively
studied~\cite{aharon2006k,mairal2009online}, sometimes under the name sparse
component analysis (SCA) in the signal processing community~\cite{naini2008estimating, gribonval2010sparse}.

The focus of this work is on the \emph{identifiability} of SCA in the case of an undercomplete dictionary (that is, the number of atoms is smaller than the ambient dimension) 
in a deterministic scenario and without noise.  By deterministic, we mean that we will derive conditions under which the decomposition is always essentially unique. 
We will refer to this model as low-rank SCA (LRSCA). 
To the best of our knowledge, the identifiability of LRSCA has been treated in
only two early works~\cite{Georgiev2005,aharon2006uniqueness}.  There exist other works dealing with the identifiability of dictionary
learning such as~\cite{Gribonval2015,hillar2015When}, but these do not improve the previously mentioned bounds in the low-rank
case.  
The main contribution of this paper is to provide new strong identifiability results for LRSCA. 

In Section~\ref{subsec:related} we formally introduce LRSCA and recall previous results for this problem. We also show some examples that give a geometric intuition for LRSCA. 
 In Section~\ref{subsec:related}, we show how LRSCA relates to other low-rank matrix
factorization models, and describe some  applications of both LRSCA and the proposed identifiability results. 
In Section~\ref{sec:mainres}, we prove our main results, namely Theorems~\ref{theorem:1} and~\ref{theorem:seq}. 
        Both provide sufficient conditions based on lower bounds on the number of data points
        required to guarantee identifiability. 
       Both theorems asymptotically lead to the 
       same bound and require $\mathcal{O}(r^3/\ell^2)$ data points well spread
       among the subspaces spanned by $r-1$ atoms of the dictionary, where $r$
       is the rank of the input matrix and \jc $k=r-\ell$ \fin is the maximum
       number of atoms used by each data point. \jc The integer $\ell$ is commonly
       referred to as the co-sparsity level. \fin
       Moreover, we prove that this bound is tight in the following two cases: 
       when $\ell$ is constant in which case $\mathcal{O}({r^3})$ points are sufficient to guarantee identifiability, 
       and when $\ell$ is a fixed proportion of $r$ in which case $\mathcal{O}({r})$ points are sufficient.   
        Theorem~\ref{theorem:1} is weaker than Theorem~\ref{theorem:seq}, but
        its proof is easier to derive and the bound it provides on the
        number of points can be easily computed. 
        Theorem~\ref{theorem:seq} is based on a sequential construction
        of subspaces containing the data points, and therefore requires
        construction-dependent hypotheses harder to verify. 
        We conclude the paper by discussing some directions of further research; in particular generalizing our results \jc in the presence of noise,  \fin 
        for overcomplete dictionaries, and for nonnegative coefficients.

\section*{Notations}

Vectors are denoted as small letters $x$, matrices as capital letters $M$. 
The $i$th column of matrix $B$ is denoted $b_i$, and the $j$th entry in that column is denoted $b_{i,j}$. 
The quantity 
$\| x \|_{0}$ is the so-called $\ell_0$ norm of the vector $x$ defined as the
number of non-zero entries in $x$. 
The spark of a matrix is the smallest number $p$ such that there exists a set
of $p$ columns which are linearly dependent. 
By extension, if a matrix $M \in \mathbb{R}^{p \times n}$ has rank $n$, we
define $\sparkn(M) = n+1$. Note that if $\sparkn(M)=r+1$, then
$\rank(M)\geq r$. 
Given a finite set $\mathcal{S}$, we denote its cardinality by $|\mathcal{S}|$. 
Given a set of vectors $\mathcal{X}=\{x_1,\dots, x_n\}$ or a matrix $X = [ x_1,\dots, x_n ]$, 
$\spann(X) = \spann(\mathcal{X})$ is the linear subspace spanned by $\mathcal{X}$.

\section{Formalism, previous results and geometric intuition} \label{sec:relatedworks}

Let $M$ be a real $p \times n$ matrix. The working assumption of
dictionary learning is that there exist a real matrix $D$ in
$\mathds{R}^{p \times r}$ and a sparse coefficient matrix $B$ in
$\mathds{R}^{r\times n}$ such that $M=DB$. 
In this paper, we impose a strict sparsity constraint on the coefficient matrix $B$, 
namely $\|b_i \|_{0} \leq k$ for all
$i\leq n$ for some $k < r$. This requires that each column of $B$ has at least $\ell=r-k$ zero entries. 
Furthermore, we assume that matrix $M$ admits a low-rank dictionary-based
representation, so that $r$ is the rank of $M$ and $r \leq p$. 
This leads to the following model which we will refer to as low-rank sparse component analysis (LRSCA): 
\begin{equation}
   \textbf{LRSCA: } 
   \left\{ \begin{array}{l}
    M = DB, \\ M \in \mathbb{R}^{p \times n}, 
    D \in \mathbb{R}^{p \times r}, B \in \mathbb{R}^{r \times n},  \\
   \| b_i \|_{0} \leq k < r \leq p \text{ for all } i,
   \\ \rank(M) = \rank(D) = r. 
    \label{eq:model}
    \end{array} \right. 
\end{equation} 
LRSCA~\eqref{eq:model} does not take the noise into account. 
Note that without loss of generality (w.l.o.g.), a dimensionality 
reduction step may be performed on $M$ that leads to a square/complete
dictionary learning problem with $p=r$ (this requires to premultiply $M$ by a $r$-by-$p$ matrix which does not destroy the sparsity structure of $B$). 
In other words, studying undercomplete dictionary learning in the absence of noise boils down to studying square dictionary learning. 
Hence one may assume w.l.o.g.\ that $p=r$ when analyzing~\eqref{eq:model}.

 \subsection{ LRSCA: related models and applications}%
    \label{subsec:related}

    LRSCA is a sparse low-rank matrix factorization model~\cite{shen2008sparse}, and it is also closely
 related to sparse PCA~\cite{d2005direct, zou2006sparse,
 journee2010generalized}. 
 In the literature, to obtain sparse decompositions, the
 most widely used approach is to add an $\ell_1$ norm penalty term in the objective function~\cite{zou2006sparse}, which has also been used for tensor factorization models~\cite{Caiafa2013,Lim2010}. 
 As far as we know, most of these works do not discuss the identifiability
 of the LRSCA model. Note that for tensors low-rank
 factorization models such as PARAFAC~\cite{Kolda2009}, identifiability
 is satisfied under mild
 conditions without enforcing sparsity constraint; see~\cite{Domanov2013} and the references therein. 
 
 LRSCA is also related to other constrained matrix factorization models; in particular nonnegative matrix factorization (NMF)~\cite{lee1999learning}. NMF imposes that both factors $D$ and $B$ are component-wise nonnegative which, in most cases, leads to sparse decompositions. Nonnegativity allows to interpret the factors (see below for some examples) 
 which is only meaningful if the NMF solution is essentially unique. Hence  identifiability conditions for NMF has
 been an active field of research; see the recent survey~\cite{fu2018Identifiability} and the references therein. However, as far as we know, there is no identifiability result specific for sparse NMF. 
 Our results (Theorems~\ref{theorem:1} and~\ref{theorem:seq}) would apply to sparse NMF and therefore LRSCA could be used for the following applications: 
 \begin{itemize} 
         \item Spectral unmixing: given an hyperspectral image $M$ containing pixels row-wise and reflectance spectra column-wise,
                 the dictionary $D$ contains the  
                 spectral signatures of constitutive materials (called endmembers), and coefficients matrix $B$ contains
                  the abundance of each endmember in each pixel.
                 Since only a few endmembers are present in each pixel (typically at most 5), $B$ is sparse~\cite{Bioucas-Dias2012}.
                 
         \item Document classification: given a term-by-document matrix $M$, the dictionary $D$ is a collection of topics while the coefficient matrix $B$ indicates which document discusses which topic. Since most documents do not discuss most topics, $B$ is sparse~\cite{lee1999learning}. Note that, in this case, $D$ is also sparse as most topics use only a small proportion of all the words.

         \item Audio source separation: 
         given a spectrogram $M$ of a recording of several audio sources (rows correspond to frequency and columns to time), the dictionary $D$ contains 
                 the spectra of each source and the coefficients matrix $B$
                 contains the temporal activation of each source. If \jc audio
                 sources (speakers, instruments) are not active
                 at all time in the recordings then $B$ must be
                 sparse~\cite{ozerov2010Multichannel}. \fin
 \end{itemize} 

LRSCA is also closely related to subspace clustering~\cite{vidal2011subspace}.
In fact, as already noted in the literature, the exact sparsity constraint on
the columns of $B$ is equivalent to imposing that the columns of the data
matrix $M$ belong to the union of subspaces generated by subsets of the columns
of $D$ with cardinality smaller than $k$.  Therefore, given $M$, finding a
decomposition $(D,B)$ as in~\eqref{eq:model} may be written as a subspace
clustering problem, and many algorithm solutions have been derived using this
observation~\cite{Georgiev2005,naini2008estimating,gribonval2010sparse}; see
also~\cite{elhamifar2013sparse,tsakiris2017hyperplane} for recent results and
algorithms for subspace clustering.  However, as far as we know, 
these works do not discuss the identifiability of LRSCA.

\subsection{Identifiability} \label{sec:theory}



In this paper, we focus on the following question: 
Given a matrix $M\in\mathds{R}^{p \times n}$ satisfying the LRSCA model~\eqref{eq:model}, 
is the decomposition $(D,B)$ essentially unique?  
A decomposition $(D,B)$ is said to be \emph{essentially unique} if 
for any other decomposition $(D',B')$ satisfying~\eqref{eq:model} 
we have $D=D'\Pi\matr{\Sigma}$ and $B=\matr{\Sigma}^{-1}\Pi^T B'$ for a diagonal scaling matrix
$\matr{\Sigma}$ and a permutation matrix $\Pi$. 
In the remainder of this paper, we will say that two decompositions are distinct if they cannot be obtained by permutation and scaling of one another.

\subsubsection{State-of-the-art results} \label{stateart}

Surprisingly, to the best of our knowledge, most works focusing on the identifiability of dictionary learning have not tackled directly the undercomplete case, 
although numerous algorithms have been proposed for this problem and variants; see Section~\ref{subsec:related}.  

Most recent works on dictionary learning have tackled the identifiability question 
using probabilistic approaches under various a priori distributions for the locus  and values of the non-zero entries of $B$; see~\cite{Gribonval2015} for a summary of such results.  
\jc Algorithmic recovery results are
also available and usually require strong assumptions on the entries of $B$.
For example, \cite{spielman2012exact,blasiok2016improved,adamczak2016note} show that when $k=\mathcal{O}(\sqrt{r})$, roughly
$r\log(r)$ samples are sufficient for recovery, and propose an efficient dictionary learning algorithm, referred to as Exact Recovery of Sparsely-Used Dictionaries (ER-SpUD). 
Other examples include 
a Riemannian trust-region method proposed in~\cite{sun2017complete, sun2017complete2} which assumes $\mathcal{O}(r)$ zeros per column of $B$, 
a model based on tensor Tucker decompositions~\cite{Anandkumar2015} which requires structured sparsity, and an algorithm similar to basis pursuit 
which requires separability and nonnegativity~\cite{Arora2013}. 
\fin 
These results do not consider degeneracies that could occur when studying unions of subspaces (because degeneracies happen with probability zero). This makes the study of the deterministic case rather different. 


To the best of our knowledge, only \jc three \fin
works~\cite{aharon2006uniqueness,Georgiev2005,hillar2015When} have studied specifically the
identifiability of both $D$ and $B$ under sparsity constraints in the absence
of noise and without any a priori distribution on the entries of $B$, that is, 
no assumption is made on $B$ beyond sparsity.  
The first result is due to Aharon \textit{et al.}~\cite{aharon2006uniqueness}, which states that if
\begin{itemize}
    \item $k < \frac{\sparkn(D)}{2}$, 
    \item every $k$-dimensional subspace spanned by $k$ columns of $D$ contains at
        least $k+1$ columns of $M$ whose spark is $k+1$. 
    \item no $k$-dimensional subspace contains $k+1$ columns of $M$ except
        those generated by $D$ (non-degeneracy), 
\end{itemize}
then the decomposition $(D,B)$ is essentially unique. 
This result is however somewhat not satisfying since the number
of points required in total is extremely large, proportional to ${(k+1)}{r\choose k}$,
and the non-degeneracy condition, although not restrictive in practice,
basically means that uniqueness is assumed from the start.  

\jc The second 
result by
Hillar and Sommar~\cite[Theorem 1]{hillar2015When} does not improve on the above bound in the low-rank and deterministic setting. It states that if 
\begin{itemize}
        \item $k < \frac{\sparkn(D)}{2}$,
        \item every $k$-dimensional subspace spanned by $k$ columns of $D$
                contains at least $k{r\choose k}$ columns of $M$ in a general position,
\end{itemize}
then the decomposition $(D,B)$ is essentially unique. Note that the notion of general position in the results of Hillar and Sommar
implies the conditions of Aharon \textit{et. al.} 
Again, the number of samples required is large, namely $k{r\choose k}^2$ since there are ${r\choose k}$ subspaces spanned by $k$ columns of $D$. \fin


A \jc third \fin seemingly more powerful result is due to Georgiev \textit{et al.}~\cite{Georgiev2005}. 
Among the proposed results on identifiability in their contribution, the most commonly used one is the following. If
\begin{itemize}
    \item $D$ is full column rank, 
    \item $\| b_i \|_0 = r-1$,
    \item every subspace spanned by $r-1$ columns of $D$ contains at
        least $r$ columns of $M$ whose spark is equal to $r$, 
\end{itemize}
then the decomposition $(D,B)$ is essentially unique. 
We show in the next
subsection that this result is incorrect. 



\subsubsection{Examples and lower bounds on the number of columns of $M$} 

Before exposing our identifiability results, let us provide the
reader with some geometric intuition by presenting a few simple examples.
To allow 2-dimensional representations of 3-dimensional problems ($r=3$), 
all drawings are done in a projective space, where the vectors were
normalized to have their entries summing to 1 (hence, in 3 dimensions, hyperplanes are represented by lines).  

    \begin{example}[9 points lying on 3 hyperplanes in dimension 3]  \label{examp:2} 
 In Figure~\ref{fig:e2}, we provide an
    example of a matrix $M$ which admits exactly two decompositions of the form 
    \[
    M = [d_1,d_2,d_3]
    \left[\begin{array}{ccccccccc}
    \ast & \ast & \ast & \ast & \ast & \ast & 0 & 0 & 0 \\
    \ast & \ast & \ast & 0 & 0 & 0 & \ast & \ast & \ast \\
    0 & 0 & 0 & \ast & \ast & \ast & \ast & \ast & \ast 
\end{array}\right], 
\] 
where each hyperplane generated by two columns of $D$ contains exactly three points. 
In other words, the matrix $M$ admits two LRSCA decompositions~\eqref{eq:model}
with $p=r=3$, $k=2$, $\ell=1$ and $n=9$.  
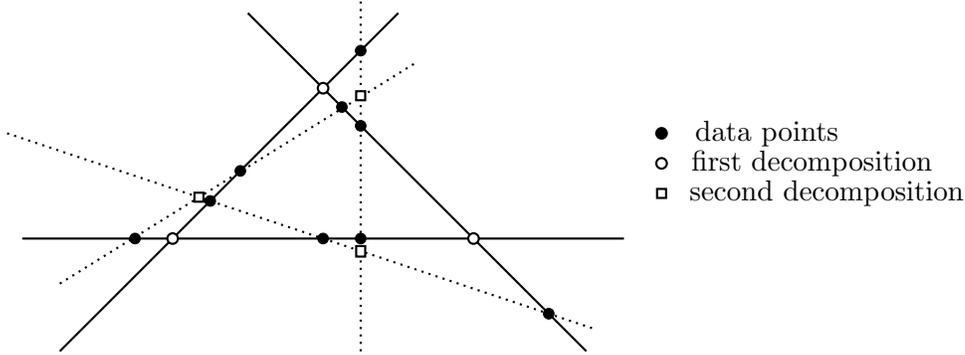
\begin{figure}[ht!]
    \centering
    \begin{tikzpicture}[scale=1]

    \draw[thick] (-1.5,-1.5) -- (3,3);
    \draw[thick] (1,3) -- (5.5,-1.5);
    \draw[thick] (-2,0) -- (6,0);
    \draw[thick,fill=white] (0,0) circle (2pt);
    \draw[thick,fill=white] (2,2) circle (2pt);
    \draw[thick,fill=white] (4,0) circle (2pt);

    \draw[fill=black] (0.5,0.5) circle (2pt);
    \draw[fill=black] (0.90,0.90) circle (2pt);
    \draw[fill=black] (2.5,1.5) circle (2pt);
    \draw[fill=black] (2.25,1.75) circle (2pt);
    \draw[fill=black] (2,0) circle (2pt);
    \draw[fill=black] (2.5,0) circle (2pt);
    \draw[fill=black] (-0.5,0) circle (2pt);
    \draw[fill=black] (2.5,2.5) circle (2pt);
    \draw[fill=black] (5,-1) circle (2pt);

    \draw[dotted,thick] (2.5,-1.5) -- (2.5,3.2);
    \draw[dotted,thick] (-1.5,-0.6) -- (3.25,2.35);
    \draw[dotted,thick] (-2.2,1.4) -- (5.6,-1.2);

    \draw[thick,fill=white] (2.44,1.84) rectangle (2.56,1.96);
    \draw[thick,fill=white] (2.44,-0.24) rectangle (2.56,-0.10);
    \draw[thick,fill=white] (0.29,0.49) rectangle (0.41,0.61);


    \draw[thick,fill=black] (6.5,1.4) circle (2pt);
    \node at (7.9,1.4) {data points};
    \draw[thick,fill=white] (6.5,1) circle (2pt);
    \node at (8.5,1) {first decomposition};
    \draw[thick,fill=white] (6.44,0.54) rectangle (6.56,0.66);
    \node at (8.70,0.6) {second decomposition};

\end{tikzpicture}
    \caption{A scenario where SCA is not unique, although it would be unique
    generically (that is, if the points were generated randomly on the hyperplanes 
    generated by the combinations of any two columns of $D$).}
    \label{fig:e2}
\end{figure} 
This provides a counter-example in the case $r=3$ to the result
from~\cite{Georgiev2005} (see Section~\ref{stateart}). 
In fact, 
 the matrix $M$ contains 9 data points with two distinct decompositions, 
although $M$ satisfies the conditions in~\cite{Georgiev2005}, 
namely, $D$ has full column rank and each hyperplane contains $3$ columns of $M$ whose spark is $3$. 
This example will generically not happen, since observing three aligned points generated randomly on three two-dimensional subspaces has probability zero. 
Hence, most low-rank SCA models with three points on each hyperplane in the case $k=2=r-1$ do not suffer from identifiability issues. 
\end{example}

Inspired by Example~\ref{examp:2}, 
for any $r$ and $k=r-1$, it is possible to construct a matrix $M$ 
with two distinct LRSCA decompositions that has $n = r^3 - 2r^2$ columns, 
with $r^2 - 2r$ columns with spark $r$ on each subspace spanned by $r-1$ columns of $D$; 
see\footnote{Algorithm~\ref{alg:ce} is available online from \url{https://sites.google.com/site/nicolasgillis/}.} Algorithm~\ref{alg:ce}. 
Lemma~\ref{lemma:alg1} proves that the construction of Algorithm~\ref{alg:ce} will generate such examples with probability one. For example, for $r=4$ and $k=3$, Algorithms~\ref{alg:ce} generates a matrix $M$ with two distinct LRSCA decompositions with $n = 32$, with 8 data points whose spark is 4 on each subspace spanned by 3 columns of $D$ . 


    \begin{algorithm}
        \caption{Generating matrix $M$ with distinct LRSCA decompositions~\eqref{eq:model} 
        with $r^3 - 2r^2$ columns in the case $k=r-1$.}  
         \label{alg:ce}
\begin{algorithmic}

    \STATE \textbf{INPUT: } An integer $r$. 
\STATE{\textbf{OUTPUT: } 
Data matrix $M \in \mathbb{R}^{r \times n}$ with $n = r^3 - 2r^2$ columns, and two decompositions 
$(D^{(1)},B^{(1)})$ 
and
$(D^{(2)},B^{(2)})$ 
for $M = D^{(1)} B^{(1)} = D^{(2)}B^{(2)}$ satisfying~\eqref{eq:model} and with $r^2 - 2r$ data points with spark $r$ on each subspace spanned by $r-1$ columns of $D^{(i)}$ ($i=1,2$).}     
    
    \STATE{ 
      1/  Generate at random two full-rank matrices $D^{(1)}$ and $D^{(2)}$ in $\mathbb{R}^{r \times r}$. 
      For example, use i.i.d.\ Gaussian distribution for the entries of
      $D^{(1)}$ and $D^{(2)}$. 
        \\}
    \STATE{
    2/ For $t=1,2$, define the $r$ hyperplanes $\mathds{F}_j^{(t)} = \spann(\{d^{(t)}_i\}_{i\neq j})$ for $j=1,2,\dots,r$. 
    \\    }
    \STATE{
    3/ Generate at random $r-2$ points on each intersection
        $\mathds{F}_j^{(1)} \bigcap \mathds{F}^{(2)}_l$ for $j,l \in [1,r]$ (there are $r^2$ such intersections), for a total of $r^2(r-2)$ points, e.g., compute a basis of the intersection and then use i.i.d.\ Gaussian distribution for the weights of the linear combinations of the $r-2$ points.   
Let the columns of $M$ consist of these $r^2(r-2)$ points.          
            \\}
    \STATE{
    4/ For $t=1,2$,  compute the coefficient matrices $B^{(t)}$ such that $M = D^{(t)}B^{(t)}$. 
    }
\end{algorithmic}
\end{algorithm}

\begin{lemma} \label{lemma:alg1}
When using the Gaussian distribution, Algorithm~\ref{alg:ce} generates with probability one a matrix $M \in \mathbb{R}^{r \times n}$ with $n=r^3 - 2r^2$ columns, where each subspace spanned by $r-1$ columns of $D$ contains $r^2-2r$ columns of $M$ that have spark $r$.  
\end{lemma}
\begin{proof}
Under the i.i.d.\ Gaussian distribution, with probability one, $D^{(1)}$ and $D^{(2)}$ are full rank and the $2r$ subspaces $\mathds{F}_j^{(t)} = \spann(\{d^{(t)}_i\}_{i\neq j})$ for $j=1,2,\dots,r$ and $t=1,2$ do not coincide.  
If the $r-2$ points on each intersection
        $\mathds{F}_j^{(1)} \bigcap \mathds{F}^{(2)}_l$ for $j,l \in [1,r]$ are generated as follows:  
        (i)~compute a basis of the intersection, and 
        (ii)~use i.i.d.\ Gaussian distribution for the weights of the linear combination in that subspace,  
        then these points have spark $r$, with probability one. 
        First, note that these intersections exist and 
        have dimension $r-2$, so that the $r-2$ points have spark $r-1$ with probability one.   
    Note that the intersections $\mathds{F}_j^{(1)} \bigcap \mathds{F}^{(2)}_l$ of dimension $r-2$ define $r^2$ subspaces that do not coincide, with probability one. 
    Note also that each subspace $\mathds{F}_j^{(k)}$ contains exactly $r(r-2)$ points whose spark cannot be larger than $r$ since it has dimension $r-1$.    
With probability one, the spark is exactly $r$: 
$r-1$ columns cannot be linearly dependant (hence have rank $r-2$) since only subsets of $(r-2)$ points were generated on the same $(r-2)$-dimensional subspace.

Finally, by construction, $B^{(t)}$ ($t=1,2$) exists and each column is at \jc
worse $k-1$ sparse \jc since every column of $M$ belongs to
$\mathds{F}_j^{(t)}$ for some $j$. In fact, it is exactly \jc $k-1$ sparse \fin with probability one since the points were picked at random in the intersections hence have non-zero coefficients in the corresponding columns of $D^{(t)}$.  
\end{proof}


Lemma~\ref{lemma:alg1} implies that, in the deterministic case and for $k=r-1$, at least $\mathcal{O}(r^3)$ points are necessary 
to guarantee essential uniqueness of LRSCA decompositions. 
Interestingly, we will prove in Theorem~\ref{theorem:seq} that adding a single point to any subspace spanned by $r-1$ columns of $D$ will make the decomposition essentially unique. 
The next example illustrates this fact in dimension 3, with a hyperplane containing 4 points.

 \begin{example}[$4+3+2$ points lying on 3 hyperplanes in dimension 3] \label{examp:3}
 Figure~\ref{fig:e3} shows an example in dimension 3 
 with 4 points on a single hyperplane (the four aligned points). It turns out that there cannot be 3 hyperplanes
covering these 4 points that do not contain the span of these 4 points.
Hence, the hyperplane containing these 4 points must be identified (on the figure, the line containing the 4 aligned points must be identified), meaning that the span of these four points must coincide with the span of two columns of $D$ in any LRSCA decomposition of $M$; 
see Lemma~\ref{lemma:2}. 
Using a similar argument, the line containing the 3 aligned points not on the first identified line must be identified (because these 3 points cannot be covered by two other lines). 
Finally, the last two points define a single line that must be identified as well. 
This implies that the decomposition is essentially unique since identifying all hyperplanes generated by $r-1$ columns of $D$ makes $D$ unique up to permutation and scaling of its columns (see Corollary~\ref{corollary:1}), while $B$ is unique since $D$ has  full column rank. 
This will be proved rigorously and for general dimensions and sparsity in Theorem~\ref{theorem:seq}. 
\begin{figure}[h!]
\centering
    \begin{tikzpicture}[scale=1]

    \draw[thick] (-1,-1) -- (3,3);
    \draw[thick] (1,3) -- (5,-1);
    \draw[thick] (-1,0) -- (5.5,0);
    
    \draw[thick,fill=white] (0,0) circle (2pt);
    \draw[thick,fill=white] (2,2) circle (2pt);
    \draw[thick,fill=white] (4,0) circle (2pt);

    \draw[fill=black] (0.5,0.5) circle (2pt);
    \draw[fill=black] (0.90,0.90) circle (2pt);
    \draw[fill=black] (2.5,1.5) circle (2pt);
    \draw[fill=black] (2.25,1.75) circle (2pt);
    \draw[fill=black] (2,0) circle (2pt);
    \draw[fill=black] (2.5,0) circle (2pt);
    \draw[fill=black] (-0.5,0) circle (2pt);
    \draw[fill=black] (2.5,2.5) circle (2pt);


    \draw[fill=black] (5,0) circle (2pt);


    \draw[thick,fill=black] (6.5,1.4) circle (2pt);
    \node at (7.9,1.4) {data points};
    \draw[thick,fill=white] (6.5,1) circle (2pt);
    \node at (8.7,1) {unique decomposition};

\end{tikzpicture}
\caption{A scenario where $M$ admits an essentially unique LRSCA~\eqref{eq:model} when $k=r-1=2$.}
\label{fig:e3}
\end{figure}
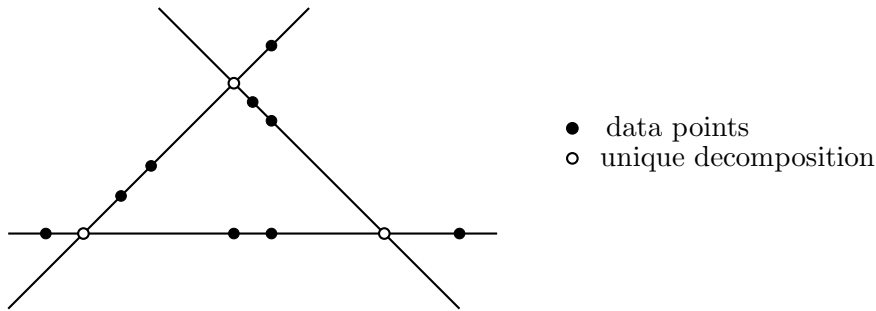 
 \end{example}

    To summarize, in this section, 
    we have made the following observations: 
    \begin{itemize}
        \item If $M$ follows an LRSCA model with sparsity set to $r-1$ where each hyperplane contains strictly less than $r$ points, then the factorization of $M$ is never unique. 
            
        \item If each these hyperplanes contain exactly $r$ points, 
        it may happen that $M$ that does not admit an essentially unique decomposition, 
        although this is unlikely in practice. 
        In fact, these hyperplanes can contain up to $r(r-2)$ points while identifiability is not guaranteed; see Algorithm~\ref{alg:ce} and Lemma~\ref{lemma:alg1}. 
    \end{itemize}

       \section{Identifiability of LRSCA}\label{sec:mainres}

In this section, we prove our main identifiability results for LRSCA~\eqref{eq:model}. 
In Section~\ref{sec:defprop}, we provide some definitions and properties that will be useful for our purpose. 
Although these properties are well-known, we provide the proofs to have the paper self contained. 
In Section~\ref{sec:maididres}, we prove our main theorems and discuss the tightness of our bounds on the number of data points needed to guarantee identifiability.

    \subsection{Definitions and properties} \label{sec:defprop}

Let us define the hyperplanes spanned by the columns of a matrix $D$, and a covering of a set of points by a set of hyperplanes. 

\begin{definition}
Given $D\in\mathds{R}^{p \times r}$ with rank $r$, we refer to the following subspaces 
\begin{equation}
    \mathds{F}_i(D) = \spann\left( \left\{d_j\right\}_{j\neq i} \right) \quad 1 \leq i \leq r, 
\end{equation}
as the $r$ hyperplanes generated by $D$. 
To simplify notation and when it is clear from the context, we will drop the
argument and simply write $\mathds{F}_i = \mathds{F}_i(D)$. 
Note that when $p > r$, $\mathds{F}_i$ are not hyperplanes but ($r-1$)-dimensional subspaces in $\mathbb{R}^p$. However, as explained in the beginning of Section~\ref{sec:relatedworks}, 
we could assume w.l.o.g.\ that $p=r$ which justifies this slight abuse of language ($\mathds{F}_i$'s are hyperplanes within $\spann(D)$). 
\end{definition}

\begin{definition}
    A set of subspaces $\left\{\mathds{F}_i\right\}_{i=1}^r$ 
    forms a covering set of the data points
    $\{m_j\}_{j = 1}^{n}$ if and only if for all $j=1,2,\dots,n$  
    there exists $1 \leq i \leq r$ such that $m_j \in \mathds{F}_i$. 
\end{definition}

Hyperplanes generated by $D$ in LRSCA~\eqref{eq:model} have the following properties.

\begin{lemma} \label{lemma:1}
    Let $M=DB$ follow the LRSCA model~\eqref{eq:model},  
    and let 
    $\{\mathds{F}_j\}_{j=1}^r$ be the hyperplanes generated by $D$. 
    Then, 
    \begin{enumerate}
        \item Hyperplanes $\mathds{F}_j$'s are distinct $r-1$ dimensional subspaces.
        \item Matrix $D$ is uniquely defined by its hyperplanes $\{\mathds{F}_j\}_{j=1}^r$ up to scaling and permutation: 
            \begin{equation}
                d_i \in \bigcap_{j\neq i} \mathds{F}_j \quad 1 \leq i \leq r,
                \label{eq:facets2A} 
            \end{equation}
            where $\bigcap_{j\neq i} \mathds{F}_j$ are 1-dimensional subspaces. 
            
        \item The set $\{\mathds{F}_j\}_{j\leq r}$ is a covering of the columns of $M$.
    \end{enumerate}
\end{lemma}

\begin{proof}
   The claims 1.~and 2.~follow directly from standard linear algebra and the LRSCA model~\eqref{eq:model}, since $\rank(D) = r$. 
   Claim 3.\ follows from the sparsity of $B$: since $||\vb_i||_0 = k < r$
   for all $i$, all points $m_i= Db_i$ belong to at least one hyperplane $\mathds{F}_j$ (namely, $j$ is such that $b_{i,j}=0$). 
\end{proof}

Using the Lemma above, one can easily link the essential uniqueness of an LRSCA decomposition 
and the uniqueness of the hyperplanes generated by $D$.  

\begin{corollary}\label{corollary:1}
Let $M=DB$ follow the LRSCA model~\eqref{eq:model}. 
Then the following are equivalent: 
    \begin{enumerate}
        \item $(D,B)$ is essentially unique.
        \item There is a unique covering set of $M$
            involving $r$ subspaces of dimension $r-1$. 
    \end{enumerate}
\end{corollary}
\begin{proof}
This follows directly from Lemma~\ref{lemma:1}. 
In fact, given an essentially unique LRSCA decomposition, 
there is a unique set of hyperplanes that contains the
columns of $M$.  
Reversely, given a unique set of hyperplanes, a unique factor $D$ is obtained.
Since the data is contained in the union of these hyperplanes, the factor $B$ exists. It
can be uniquely determined knowing $D$ and $M$, since $\rank(D) = r$. 
\end{proof} 

In summary, studying the identifiability of LRSCA~\eqref{eq:model} is equivalent to studying the uniqueness of $r$ hyperplanes which union contains all the data points.

\subsection{Main identifiability results} \label{sec:maididres}


The following key lemma provides a condition under which a hyperplane generated by $D$ spanned by
the data points contained in the submatrix $M^{(1)}$ of $M$ has to be a hyperplane of any LRSCA decomposition of $M$. 
In other terms Lemma~\ref{lemma:2} focuses on the identifiability of a single
hyperplane generated by $r-1$ columns of $D$. 
Theorem~\ref{theorem:1} will use this result to uniquely identify all
hyperplanes generated by $D$ implying that $D$ is uniquely identified
(Lemma~\ref{lemma:1}). Theorem~\ref{theorem:seq} refines the results of
Theorem~\ref{theorem:1}, but features more involved technical conditions, which are relaxed in Corollary~\ref{corollary:2}.

\begin{lemma}\label{lemma:2}
    Let $M^{(1)}\in\mathds{R}^{p \times n_1}$ be a set of $n_1$ data points lying on a $r-1$
    dimensional subspace, that is, $\rank(M^{(1)})=r-1$, with  
    $\sparkn(M^{(1)}) = r$. 
    Let $D\in\mathds{R}^{p \times r}$,
    $B^{(1)} \in\mathds{R}^{r\times n_1}$ and $k<r$ be such that $M^{(1)}=DB^{(1)}$,
    $D$ is full column rank and $\|b_j^{(1)}\|_0\leq k$ for
    all $j$. Then the following holds:
    \begin{equation}
        n_1\geq \left\lfloor \frac{r(r-2)}{r-k}\right\rfloor +1 \quad \Rightarrow \quad             
        \text{there exists } j\text{ such that }\mathds{F}_j(D) = \spann(M^{(1)}) .   
    \end{equation}
\end{lemma}

\begin{proof}
Let us define $S_j =\{ m^{(1)}_i \ | \ m^{(1)}_i \in  \mathds{F}_j \}$ the set of columns
of $M^{(1)}$ contained in the $j$th hyperplane $\mathds{F}_j$ generated by $D$. 
If there exists some $j$ such that 
$|S_j| \geq r-1$, 
then $\spann(S_j)= \mathds{F}_j$ since, by assumption, $\sparkn(M^{(1)})=r$ and $\dimn(\mathds{F}_j) = r-1$. 
    Because $\rank(M^{(1)}) = r-1$, this implies that an
    hyperplane $\mathds{F}_j$ containing strictly more than $r-2$ data points
    of     $M^{(1)}$ satisfies $\mathds{F}_j = \spann(M^{(1)})$. 
   Moreover, every column of $M^{(1)}$ lies on at least $r-k$ hyperplanes since $\|b_i\|_0 \leq k$ for all $i$.   
Hence one can check that the maximum number of columns that $M^{(1)}$ can contain such that each
   of the $r$ hyperplanes generated by $D$ contains at most $r-2$ columns of $M^{(1)}$ 
 is given by 
    \[
        n_{\max} = \left\lfloor  \frac{r(r-2)}{r-k} \right\rfloor. 
    \]
    Hence $n_1 \geq n_{\max}+1  > n_{\max}$ implies  $\mathds{F}_j = \spann(M^{(1)})$ for some $j$,    which completes the proof. 
\end{proof}

Lemma~\ref{lemma:2} provides a bound on the number of points on a hyperplane that ensures that this hyperplane is contained in any LRSCA decomposition of a matrix containing these points.  
For instance, for $r=3$ and $k=2$, a hyperplane containing $r(r-2)+1 = 4$ columns of $M$ with spark 3 must be included in any LRSCA decomposition of $M$, while if it contains only 3 columns, 
it is not necessarily the case; see Examples~\ref{examp:2} and~\ref{examp:3}.  
For $r=4$ and $k=3$, 9 columns or more of $M$ belonging to a $r-1$ dimensional subspace
and having spark $r$ implies that the corresponding hyperplane will be contained in any 3-sparse decomposition of $M$. If such a hyperplane contain only 8 columns, its identifiability is not guaranteed; see the construction in Algorithm~\ref{alg:ce}. 

We can now use Lemma~\ref{lemma:2} to provide a sufficient condition to the uniqueness of LRSCA~\eqref{eq:model} by applying Lemma~\ref{lemma:2} to each hyperplane of a decomposition of $M$. 


\begin{theorem}\label{theorem:1}
    Let $M = DB$ satisfy the LRSCA model~\eqref{eq:model}. 
    The decomposition $(D,B)$ is essentially unique if there exists a collection of subsets
    $\{I_j\}_{j = 1}^r$ such that $M^{(j)} = M(:,I_j)$ satisfies the following conditions: 
   every column of $M^{(j)}$ belongs to the $j$th hyperplane generated by $D$, that is, 
   $B(j,I_j) = 0$ for all $j$, and for all $j$ 
   \begin{equation} \label{cond:theorem1}
        \sparkn(M^{(j)}) = r 
        \quad
        \text{ and }  
        \quad 
      |I_j| \geq \left\lfloor \frac{r(r-2)}{r-k} \right\rfloor +1.
    \end{equation}
\end{theorem}
\begin{proof}
By assumption, $\mathds{F}_j = \spann(M^{(j)})$. 
Then, uniqueness of $(B,D)$ follows directly from Corollary~\ref{corollary:1} (it is equivalent to identify $D$ or its hyperplanes) 
and Lemma~\ref{lemma:2} which implies that all hyperplanes 
generated by $D$ are identified under the condition~\eqref{cond:theorem1}.  
\end{proof}

A simpler way to phrase Theorem~\ref{theorem:1} is the following: an LRSCA decomposition $M=DB$ is essentially unique if on each subspace spanned by all but one column of $D$, 
there are $\left\lfloor \frac{r(r-2)}{r-k} \right\rfloor +1$ data points with spark $r$.

\subsubsection{Tightness of Theorem~\ref{theorem:1}} \label{tightnessth1}

Theorem~\ref{theorem:1} can be used to compute a minimum value of the number $n$ of columns of a matrix $M$ satisfying the assumptions of Theorem~\ref{theorem:1}.  
Since each column of $M$ belongs to $r-k$ hyperplanes of $D$, it may belong to $r-k$ subsets $I_j$'s hence 
the condition~\eqref{cond:theorem1} implies that   
\[
n \geq \frac{\sum_{j=1}^r |I_j|}{r-k} \geq \frac{r}{r-k} \, \left(\left\lfloor \frac{r(r-2)}{r-k} \right\rfloor  +1 \right).  
\]    
 For $k=r-1$, the bound gives $n \geq r^3-2r^2+r$ which is tight up to a constant $r$ since Algorithm~\ref{alg:ce} provides distinct LRSCA decompositions with $n = r^3-2r^2$ 
satisfying the conditions $\rank(M^{(j)}) = r-1$ and  $\sparkn(M^{(j)}) = r$
with $|I_j| = r(r-2)$. 
For $k=1$, clearly $|I_j| = 1$ for all $j$ ensures uniqueness hence $n=r$ is enough. 
Our bound can be simplified as follows: since $(r-1)^2 > r(r-2)$, 
we have $\left\lfloor \frac{r(r-2)}{r-1} \right\rfloor = r-2$ hence
\[
\frac{r}{r-1} \, \left(\left\lfloor \frac{r(r-2)}{r-1} \right\rfloor  +1 \right)
= 
\frac{r}{r-1} \, \left( r-1 \right)
= r, 
\] 
hence is tight. 


If the columns of the matrix $B$ contain a number of zero entries proportional
to $r$, that is, \jc if the co-sparsity level satisfies \fin $\ell=\alpha r$ for some constant $\alpha \in [1/r,(r-1)/r]$ so that $k = r(1-\alpha)$, our bound requires $n \geq \frac{r-2+\alpha}{\alpha^2}$ which is proportional
 to $r$ hence is tight up to a (possibly large) constant factor $\frac{1}{\alpha^2}$ (since clearly $n \geq r$ is a
 necessary condition for essential uniqueness). 
 This is rather interesting to observe since in many practical problems the
 sparsity is often proportional to $r$ (e.g., 10\% of zero entries).  
This is interesting to put in perspective with the $\mathcal{O}\left(r\log(r)\right)$ 
data points required in a probabilistic setting, 
when columns of $B$ are generated randomly, which can be seen as an instance of
the coupon collector problem as shown by Spielman \textit{et al.}~\cite{spielman2012exact}.

\begin{remark} 
Note that Theorem~\ref{theorem:1} can be combined with probabilistic arguments to obtain probabilistic bounds. 
However such bounds would be weaker for example than the one of Spielman \textit{et al.}~\cite{spielman2012exact} because Theorem~\ref{theorem:1} does use the assumption that the data points are sampled in general position.  
For example, let us consider the case $k=r-1$ and assume that the coefficients are sampled from the product of a Bernouilli distribution (for placing zeros) and a Uniform distribution (for sampling non-zeros). 
We need to have $m=\mathcal{O}(r^2)$ points in each subspace. 
The coupon collector problem requires on average $O(r \log(r))$ points to have a single point on each of the $r$ hyperplanes. Hence, to have $m$ points on each hyperplane, we need on average to sample at most $O( m r \log(r) )$ points (this can be proved using the linearity of the expected value), 
that is, $O( r^3 \log(r))$ points, which is, 
up to the factor $\log(r)$, asymptotically the same as our deterministic bound. 
Note that slightly thighter bounds can be obtained using a more refined analysis of the generalized coupon problem that needs to collect $m$ coupon of each type, also known as the double dixie cup problem~\cite{newman1960double}. 
\end{remark} 

 For $\alpha=1/r$ (resp.\ $(r-1)/r$), that is, $\ell=1$ and $k=r-1$, 
 we recover the bound $n = \Omega(r^3)$ (resp.\ $n \geq r$). 
For other values, it is more difficult to prove tightness of the bound. 
For example,  for  a number of zero entries in the columns of $B$ proportional
to $\sqrt{r}$, that is, $\alpha = 1/\sqrt{r}$, we get $n = \Omega(r^2)$ which
we do not know whether it is tight up to a constant factor. Generalizing the construction of Algorithm~\ref{alg:ce} 
will require to intersect $\ell$ hyperplanes among the $\mathds{F}_j^{(1)}$'s with
$\ell$ hyperplanes among the $\mathds{F}_j^{(2)}$'s (to guarantee $B^{(1)}$ and
$B^{(2)}$ to be \jc $k$-sparse): \fin  
there are many such intersections and it is not clear how to count the points that can be generated to avoid degeneracy, that is, to guarantee that the spark the the data points on each hyperplane generated by $D$ is $r-1$  
(note that $\ell < r/2$ is required for these intersections to be non-empty). 
Proving the tightness of the bounds in these cases is a direction of further research. 



\paragraph{Brute-force algorithm for LRSCA}  
 As a by-product of the proposed
identifiability result, we also exhibit an algorithm that, under the conditions of Theorem~\ref{theorem:1}, outputs the unique solution to the LRSCA problem in a finite number of steps; this is similar what was proposed in~\cite{Georgiev2005, aharon2006uniqueness}. 
This algorithm iterates the following steps until either all
data points belong to an identified hyperplane (in which case the algorithm has successfully found the unique solution) or when all different possible combinations of data points are checked: 
\begin{enumerate}
        \item Choose $p=\left\lfloor \frac{r(r-2)}{r-k} \right\rfloor +1$ data points in $M$. 
        \item Check if these data points satisfy the spark condition. If not, return to step~1. 
        \item Check if these data points belong to an $(r-1)$-dimensional subspace. If this is the case, then their span must be one of the hyperplanes $\mathcal{F}_j(D)$. 
\end{enumerate}
This provides a practical way, albeit computationally impractical with up to $\binom{n}{p}$ combinations of points to check, 
to assert the existence of a solution to LRSCA and find the
unique solution under the assumptions of Theorem~\ref{theorem:1}.

\subsubsection{Stronger identifiability result} 

\jc Theorem~\ref{theorem:1} is a significant improvement to already known
sufficient identifiability conditions for complete dictionary learning.
However, 
it does not allow to show identifiability of Example~\ref{examp:3} ($r=3$,
$k=2$) in Section~\ref{sec:theory}. In fact, we have shown that only $n=9$ points may be sufficient for LRSCA to be identifiable, 
while Theorem~\ref{theorem:1} requires $n\geq 12$. 
In what follows, we
propose Theorem~\ref{theorem:seq} featuring somewhat involved
conditions but weaker than Theorem~\ref{theorem:1}. 
This allows us to obtain Corollary~\ref{corollary:2} which conditions are as
simple as that of Theorem~\ref{theorem:1} while reducing roughly by half the sample size required by Theorem~\ref{theorem:1} in the particular case $k=r-1$. \fin 

The idea is the following: 
in Theorem~\ref{theorem:1}, 
we have identified each hyperplane independently of the others. 
However, the fact that a hyperplane is identified influences the bound on the
number of points needed on the other hyperplanes; 
see Example~\ref{examp:3} where the second (resp.\ third) hyperplane only needs to contain 3 points (resp.\ 2 points) to be identifiable, instead of 4 as required by Theorem~\ref{theorem:1}.   
Hence using Lemma~\ref{lemma:2} sequentially, instead of simultaneously, for all hyperplanes, the following stronger result can be derived. 
\begin{theorem}\label{theorem:seq}
    Let $M = DB$ satisfy the LRSCA model~\eqref{eq:model}.  
    The decomposition $(D,B)$ is essentially  unique if 
    there exists a collection of subsets 
    $\left\{ I_j\right\}_{j=1}^{r}$ such that $M^{(j)}=M(:,I_j)$ satisfies the following
    conditions: 
\begin{enumerate}
 \item Every column of $M^{(j)}$ belongs to the jth hyperplane $\mathds{F}_j$
     of $D$, that is, for all~$j$, $B(j,I_j)=0$.
    
 \item We have for all $j$  
    \begin{equation}  \label{cond:theoremseq} 
         \begin{array}{l}
     \sparkn(M^{(j)})=r 
     \quad \text{ and } \quad 
      |I_j|
        \geq \left\lfloor \frac{(r-j+1)(r-2) + \sum_{i\in I_j} c_{i \rightarrow j} }{r-k} \right\rfloor +1 , 
        \end{array}
     \end{equation} 
 where the quantity 
 \[
 c_{i \rightarrow j} 
 = \left| \left\{ p \, | \, p<j, m_i \in \mathds{F}_p  \right\} \right| 
 = \left| \left\{ p \, | \, p<j, B(p,i) = 0  \right\} \right| 
 \]
 is defined for all $1 \leq j \leq r$ and all $i\in I_j$. 
 Given a point $m_i$ that belongs to the hyperplane $\mathds{F}_j$, 
 the quantity $c_{i \rightarrow j}$ is the number of hyperplanes $p$ preceding $j$ (that is, $p < j$) to which $m_i$ belongs to. 
  Note that the order in which we sort the hyperplanes plays a crucial role, as opposed to Theorem~\ref{theorem:1}, 
 since $c_{i \rightarrow j}$ depends on this ordering.  

\end{enumerate} 
\end{theorem}

Before we give the proof, let us try to shed some light on the conditions of Theorem~\ref{theorem:1}. 
The quantity $\sum_{i\in I_j} c_{i \rightarrow j}$ is smaller than $(j-1)(r-2)$: 
In fact, the intersection of two hyperplanes has dimension $r-2$ hence one cannot pick more than $r-2$ points from hyperplane $j$ in hyperplane $k \neq j$ because of the spark condition $\sparkn(M^{(j)})=r$. 
Therefore, the condition~\eqref{cond:theoremseq} of Theorem~\ref{theorem:seq} is weaker than condition~\eqref{cond:theorem1} of Theorem~\ref{theorem:1}.  


Let us consider the case $k = r-1$ and assume $c_{i \rightarrow j} = 0$ for all $i,j$ (this can actually be assumed w.l.o.g., see Lemma~\ref{lem:cij}) hence the last requirement of the third inequality in condition~\eqref{cond:theoremseq} becomes $|I_j| \geq  (r-j+1)(r-2) +1$. 
For $r=3$ and $k=2$, we know that having 4 points with spark 3 on each hyperplane is enough for identifiability (Theorem~\ref{theorem:1}).   
From Theorem~\ref{theorem:seq}, 
we know that the following weaker conditions will be enough:  
(i)~4 points with spark 3 on a first hyperplane as in Theorem~\ref{theorem:1}, 
(ii)~3 points with spark 3 on a second hyperplane that do not belong to the first hyperplane, and 
(iii)~2 points with spark 3 on a third hyperplane that do not belong to the first two hyperplanes. 
Figure~\ref{examp:3} illustrates such a unique decomposition. 


\begin{proof}[Proof of Theorem~\ref{theorem:seq}] Let us prove the result by induction.  

    First, consider the first set of indices $I_1$ satisfying~\eqref{cond:theoremseq}. 
    Since no hyperplane has already been identified, we have $c_{i \rightarrow 1} = 0$ for all $i\in I_1$. By applying Lemma~\ref{lemma:2} on $I_1$, as in Theorem~\ref{theorem:1}, 
    the first hyperplane $\mathds{F}_1$ must be contained in any decomposition of $M$. 

    Now suppose that $j-1$ hyperplanes are correctly identified where $j \geq 2$, that is, 
    the hyperplanes $\mathds{F}_{p}$ for $p < j$ 
    are identified (that is, they must belong to any decomposition).  
    Assume $M$ admits (at least) two different decompositions: one
    corresponding to the sought hyperplanes $\mathds{F}_{p}$ 
    ($1 \leq p \leq r$) and one with hyperplanes 
     $\mathds{F}'_{p}$ ($1 \leq p \leq r$). Since the first $j-1$ hyperplanes are correctly identified, by induction, 
     $\mathds{F}'_{p} = \mathds{F}_{p}$ for $p \leq j-1$ 
     (w.l.o.g.\ we assume the first $j-1$ hyperplanes $\mathds{F}'_{p}$'s correspond to the first $j-1$ hyperplanes $\mathds{F}_{p}$'s, otherwise we reorder them accordingly).   
    If $\mathds{F}_{j} = \mathds{F}'_{k}$ for some $k \geq j$, $\spann(M^{(j)})$ is a common hyperplane of both decompositions, and we move to the next $j$. 
    Otherwise $M^{(j)}$ is covered by the set of $r$ hyperplanes
    $\{\mathds{F}'_{l}\}_{l\in[1,r]}$
    that does not contain $\spann(M^{(j)})$.
    Since $\mathds{F}'_{p} = \mathds{F}_{p}$ for $p < j$, the
    hyperplanes $\mathds{F}'_p$ can be divided in two classes: hyperplanes that are identified ($p < j$) and hyperplanes that are free ($p \geq j$), that is, that are not necessarily identified. 

    Hyperplanes that are free may each contain only up to $r-2$ columns of
    $M^{(j)}$, by the same argument as in Lemma~\ref{lemma:2}: 
    In fact, if say $\mathds{F}'_p$ for some $p\geq j$ contains 
    $r-1$ columns of $M^{(j)}$, then because $\sparkn(M^{(j)})=r$, 
    the span of these $r-1$ columns equals $\spann(M^{(j)})$, 
    which contradicts our hypothesis that the hyperplane corresponding to $\spann(M^{(j)})$ is not identified. 
    Therefore, the maximal number of columns of $M^{(j)}$ on all the free
    hyperplanes is $(r-j+1)(r-2)$. 

    The identified hyperplanes $\mathds{F}_{p}$ for $p<j$ may also contain columns of $M^{(j)}$. By definition of $c_{i \rightarrow j}$, 
    the number of times the identified hyperplanes touch points of $M^{(j)}$ in $I_j$ is given by $\sum_{i \in I_j} c_{i \rightarrow j}$.      
    Hence, the total number times all hyperplanes touch points in $I_j$ is at most 
    \[
         (r-j+1)(r-2) + \sum_{i\in I_j}  c_{i \rightarrow j}. 
    \]
    Now, since the $r$ hyperplanes $\mathds{F}'_p$ must correspond to a valid decomposition, that is, a decomposition that is $r-k$ sparse, each column of $M^{(j)}$ belongs to at least $r-k$ hyperplanes.  
    This allows us to conclude: 
    the total number times the hyperplanes $\{\mathds{F}'_{l}\}_{l\in[1,r]}$ touch points in $M^{(j)}$ must be $(r-k)|I_j|$  while it is at most  $(r-j+1)(r-2) + \sum_{i\in I_j}  c_{i \rightarrow j}$. 
    Therefore, 
    \[
        | I_j | > \frac{ (r-j+1)(r-2) + \sum_{i\in I_j}  c_{i \rightarrow j}}{r-k}
    \]
    leads to a contradiction and the hyperplane $\mathds{F}_j$ must be identified. 
    
        Finally, since all hyperplanes generated by $D$ have been identified, $(B,D)$ is essentially unique (Corollary~\ref{corollary:1}) which concludes the proof. 
     \end{proof}

  \paragraph{Tightness of Theorem~\ref{theorem:seq}} 
  Since Theorem~\ref{theorem:seq} is stronger than Theorem~\ref{theorem:1}, 
  it is also tight when $k=1$, and 
  asymptotically tight when $k$ is a fixed proportion of $r$ (see the discussion after Theorem~\ref{theorem:1}). 
However, Theorem~\ref{theorem:seq} has the advantage to be tight for $k=r-1$ (which is not the case of Theorem~\ref{theorem:seq} which is tight only up to a factor $r$). 
We will see in Lemma~\ref{lem:cij} that we may assume w.l.o.g.\ that $c_{i \rightarrow j} \leq r-k-1$ hence $c_{i \rightarrow j} = 0$ for $k=r-1$. 
Therefore, when $k=r-1$, the condition on the cardinality of the sets $I_j$'s simplifies to 
$|I_j| \geq (r-j+1)(r-2) +1$. 
Using a similar construction as in Algorithm~\ref{alg:ce}, 
we can generated matrices with two distinct LRSCA decompositions 
that satisfies all the conditions of Theorem~\ref{theorem:seq} with 
$|I_j| = (r-j+1)(r-2) +1$ for all $j$ except for some $k$ for which $|I_k| = (r-k+1)(r-2)$. (Note that the two decompositions will have $k-1$ hyperplanes in common).

Although Theorem~\ref{theorem:seq} has a more relaxed condition 
        than Theorem~\ref{theorem:1} 
        on the cardinalities of the index sets $I_j$'s, 
        it is more difficult to check as we do not know a priori the values of the quantities $c_{i \rightarrow j}$ (except for $k=r-1$; see Lemma~\ref{lem:cij}).  
        In order to simplify this condition in Theorem~\ref{theorem:seq}, 
        we derive an upper bound for the quantities
        $c_{i \rightarrow j}$. 
        This will lead to a weaker identifiability condition than Theorem~\ref{theorem:seq} but easier to write down. 
 
 \begin{lemma} \label{lem:cij} 
 In Theorem~\ref{theorem:seq}, one may assume w.l.o.g.\ that 
 	$c_{i \rightarrow j} \leq r-k-1$  for all $1 \leq j \leq r$ and $i \in I_j$. 
 \end{lemma}
\begin{proof}
    Let $M=DB$ satisfy the conditions of Theorem~\ref{theorem:seq}, and
    suppose that there exist $j$ and $p \in I_j$ such that 
    $c_{p \rightarrow j} \geq r-k$. 
    Because of the spark condition on $M^{(j)}$, we have $|I_j| \geq r-1$. 
    Let us consider two cases. 
    
    \noindent \emph{Case 1: $|I_j| \geq r$.}  
    In this case, we replace $I_j$ with $I_j' = I_j \setminus{\{p\}}$ 
    and $M^{(j)}$ with $M'^{(j)} = M(:,I_j')$ in Theorem~\ref{theorem:seq}. The two conditions 
    \begin{itemize}
    
       \item[(i)]  $\sparkn(M'^{(j)}) = r$, and 
       
       \item[(ii)]  $|I_j'| > \frac{(r-j+1)(r-2) + \sum_{i \in I_j'}c_{i \rightarrow j}}{r-k}$. 
        
    \end{itemize} 
    of Theorem~\ref{theorem:seq} are still satisfied. 
    In fact, since $\frac{c_{p \rightarrow j}}{r-k} \geq 1$ and by the condition on $|I_j|$, we have $|I_j'| = |I_j| -  1$ while 
    \[
    |I_j| - 1 
        > \frac{(r-j+1)(r-2) + \sum_{i \in I_j} c_{i \rightarrow j} }{r-k} - 1 
        \geq \frac{(r-j+1)(r-2) + \sum_{i\in I_j'} c_{i \rightarrow j} }{r-k} 
    \]
    hence $I_j'$ satisfies (ii).   
    Then, because $M'^{(j)}$ has all columns
    of $M^{(j)}$ except for $m_{p}$, 
    the spark of $M'^{(j)}$ can only decrease by one if $|I_j| = r-1$ (that is, if $M'^{(j)}$ has only $r-1$ columns); see the definition of the spark. 
   Since we assume $|I_j| \geq r$ in this case, this is not possible. 
  
  \noindent \emph{Case 2: $|I_j| = r-1$.}  In this case, we will construct  
  $I_j' = I_j \backslash \{p\} \cup \{h\}$ by showing that there exists a point $m_h$ with $c_{h \rightarrow j} \leq r-k-1$ and so that that two conditions (i)-(ii) are satisfied. 
  By Theorem~\ref{theorem:seq}, the LRSCA decomposition of $M$ is unique. 
  We have $|I_j| = r-1$ and $c_{p \rightarrow j} \geq r-k$ for some $p$.  
  Let us try to construct a new factorization $(D',B')$ where we replace the single hyperplane $\mathds{F}_j = \mathds{F}_j(D)$ with a different hyperplane $\mathds{F}_j' = \mathds{F}_j(D')$ defined as 
    \[
        \mathds{F}_j' = \spann( [M'^{(j)}, \zeta] ) 
    \]
    where $\zeta$ is any vector not in $\spann( M'^{(j)} )$ 
    and such that $\mathds{F}_j' \neq \mathds{F}_j$.     
    By assumption, this cannot correspond to a valid decomposition $(D',B')$ of the form \eqref{eq:model}, hence there exists a column of $M$, say $m_h$, such that the corresponding column of $B'$ does not contain $r-k$ zeros. 
Recall $B(j,i) = 0$ if and only if the $i$th data point belongs to the $j$th hyperplane generated by $D$. 
 We cannot have $h=p$ since $m_p$ belongs to $r-k$ hyperplanes other than
 $\mathds{F}_j$ (since $c_{p \rightarrow j} \geq r-k$)  which have not been modified in the new decomposition $(D',B')$, 
otherwise the corresponding column of $B'$ has at least $r-k$ zeros. 
We must have $m_h \in \mathds{F}_j$ since $\mathds{F}_j$ is the only modified
hyperplane in the new decomposition, otherwise $b'_h$ contains at least $r-k$ zeros.  
 For the same reason, we must have $m_h \notin \spann(M'^{(j)}) \subset \mathds{F}_j'$. 
    Also, $m_h$ belongs to exactly $r-k$ hyperplanes of $D$: 
    if it belongs to $r-k+1$ hyperplanes then it belongs to at least $r-k$ in the new decomposition since only one hyperplane is modified.   
    Finally, for the decomposition to be unique, there must exist a point $m_h
    \in \mathds{F}_j$ with $c_{h \rightarrow j} \leq r-k-1$ such that $I_j' =
    I_j \backslash \{p\} \cup \{h\}$ 
    satisfies the spark condition, since $m_h \notin \spann(M'^{(j)})$.  
\end{proof}

When $k=r-1$, an interesting implication of Theorem~\ref{theorem:seq} and Lemma~\ref{lem:cij} is that points lying on the $k$th hyperplane $\mathds{F}_k$ (corresponding to the subset $I_k$) 
are useless to identify the hyperplanes $\mathds{F}_j$ for $j > k$ and can therefore be discarded. 
Hence this motivates and justifies the use of sequential algorithms that work as  follows: until $r$ hyperplanes have not been identified, 
(i) identify sufficiently many points lying on a hyperplane, 
(ii) remove all these points and go back to (i).  
Such algorithms have already been used for subspace clustering; 
see for example~\cite{vidal2011subspace} and the references therein.  


        \begin{corollary}\label{corollary:2}
 Let $M=DB$ satisfy~\eqref{eq:model}. 
    The factorization $(D,B)$ is essentially  unique under the same conditions as
    in Theorem~\ref{theorem:seq} except for the cardinalities of the index sets
    $I_j$'s which are replaced with 
 \begin{equation}\label{cond:corollary}
     |I_j| > (r-j+1) (r-2), 
 \end{equation} 
 instead of the second condition in~\eqref{cond:theoremseq}.  
 \end{corollary}
    \begin{proof}
    Lemma~\ref{lem:cij} shows that one can assume  w.l.o.g.\ in Theorem~\ref{theorem:seq} that $c_{i \rightarrow j} \leq r-k-1$  for all $1 \leq j \leq r$ and $i \in I_j$. 
    Hence the second condition in~\eqref{cond:theoremseq} which is equivalent to 
    \[
    |I_j|
        > \frac{(r-j+1)(r-2) + \sum_{i\in I_j} c_{i \rightarrow j} }{r-k} 
    \]
    is implied by 
    \begin{equation} \label{eq:inter}
    |I_j|
        >  \frac{(r-j+1)(r-2) + |I_j| (r-k-1) }{r-k}   , 
    \end{equation}
    since $\sum_{i\in I_j} c_{i \rightarrow j} \leq |I_j| (r-k-1)$. 
    After simplifications, \eqref{eq:inter} is equivalent to  
    \[
    |I_j| > (r-j+1)    (r-2),
    \] 
    which completes the proof. 
\end{proof}

By Lemma~\ref{lem:cij}, the bounds of Corollary~\ref{corollary:2} coincide with Theorem~\ref{theorem:seq} when $k=r-1$ hence is tight. For smaller $k$, Corollary~\ref{corollary:2} is weaker. In particular, for $k=1$, it is much weaker since Corollary~\ref{corollary:2} would require $\mathcal{O}(r^2)$ data points instead of only $\mathcal{O}(r)$.

\section{Conclusion and Further Work}

In this paper, we have considered the identifiability of
LRSCA~\eqref{eq:model}, which is the SCA model with an \jc (under)complete
\fin dictionary, without noise and in a deterministic setting. 
We provided new sufficient results 
(Theorems~\ref{theorem:1} and~\ref{theorem:seq}) 
 guaranteeing identifiability as soon as the data points are sufficiently numerous and well-spread on the subspaces spanned by $r-1$ atoms of the dictionary. 
 The total number of data points must be larger than $O(r^3/(r-k)^2)$ where $r$
 is the number of atoms in the dictionary and $k$ is the maximum number of non-zeros in the coefficients used to reconstruct each data point using the dictionary atoms. 
These results improve drastically on what was known previously; see the discussion in Section~\ref{stateart}. \\ 


\jc 
Further research include the derivation of identifiability results in the
presence of noise. 
Accounting for noise would be an important improvement to our
results since it would allow to identifying the dictionary $D$ and the sparse coefficients $B$ in an approximate factorization setting, which is the model used for most applications such as the ones described in Section~\ref{sec:relatedworks}. \fin   
 Further research also include the extension of our results in the case of overcomplete dictionaries. There are at least three additional issues in this scenario: 
\begin{enumerate}
\item The uniqueness of $B$ becomes non trivial as the dictionary is not full rank although sufficient conditions already exist; for example $k < \frac{\sparkn(D)}{2}$~\cite{donoho2003optimally}. 
  
\item Lemma~\ref{lemma:2} can be adapted but will require to consider $\binom{r}{d-1}$ hyperplanes generated by the $r$ columns of $D$, where $d=\rank(M)$. 
This drastically increases the number of data points needed to identify a hyperplane. 
Also, the case $\sparkn(D) < d+1$ would have to be analysed carefully since some subspaces generated by $d-1$ atoms will not be of dimension $d-1$ and will be contained in other subspaces spanned by $d-1$ atoms.  

\item All hyperplanes generated by $r-1$ atoms of the dictionary need not to be identified to make the SCA essentially unique. 
In fact, it is sufficient to identify, for each atom, $d-1$ hyperplanes containing it. 
Since each hyperplane contains $d-1$ atoms, it is sufficient to identify $r$ well-chosen hyperplanes. 
For example, for a three dimensional problem with four atoms in the dictionary
and $\ell=r-k=2$, it is sufficient to identify four well-chosen hyperplanes (namely, each atom should be at the intersection of two identified hyperplanes) among the six hyperplanes to guarantee identifiability of the atoms. 
\end{enumerate}

It would also be interesting to study the identifiability of the LRSCA model with additional constraints. 
For example, LRSCA with nonnegativity on the factor $B$ is closely related to sparse nonnegative matrix factorization~\cite{Hoyer2004} and would lead to weaker conditions on the number of points needed on each hyperplane. 
It can be shown for example that in three dimensions and \jc 2-sparse \fin
decompositions ($r=3$, $k=2$, $\ell=1$), 
three points on a hyperplane are sufficient for it to be identifiable uniquely (as opposed to four data points in LRSCA; see Lemma~\ref{lemma:2}). 
In fact, in the projective representation, the atoms are the vertices of a
triangle whose segments contain the data points: a triangle cannot contain
three aligned points unless one of its segments contains them all. Generalizing
this observation in higher dimension is a topic for further research. 

Finally, given the sequential structure of our proofs, a study of partial
identifiability, that is, on the uniqueness of some atoms in $D$ and the
related coefficients in $B$, is an interesting direction of research.
This kind of partial identifiability results could be useful for  applications where only some atoms need to be identified and interpreted.

\small 

\bibliographystyle{siamplain}
\bibliography{biblio}

\end{document}